\documentclass[12pt,reqno]{amsart}
\usepackage{amssymb}

\usepackage{amscd}

\newcommand{\RNum}[1]{\uppercase\expandafter{\romannumeral #1\relax}}

\usepackage[T2A]{fontenc}
\usepackage[utf8]{inputenc}
\usepackage[russian,english]{babel}
\input{int.def}

\usepackage{tikz-cd}
\usetikzlibrary{cd}
\usepackage{dirtytalk}
\usepackage{hyperref}
\usepackage{xcolor, enumitem}
\usepackage{centernot}

\numberwithin{equation}{section}

\usepackage[mathcal]{euscript}

\usepackage{titlesec}
\titleformat{\section}[runin]{\bfseries}{\thesection.}{3pt}{}[.]

\myitemmargin 
\usepackage{geometry}
\newgeometry{vmargin={25mm}, hmargin={20mm,20mm}}   

\usepackage{mathtools}
\usepackage{esint}

\begin{document}

\title[Analyzing Deviations of Dyadic Lines in Fast Hough Transform]%
{Analyzing Deviations of Dyadic Lines in Fast Hough Transform}

\author{Gleb Smirnov}
\address{School of Mathematics, 
University of Geneva, Rue du Conseil-G{\'e}n{\'e}ral, 7, 1205, Geneva, Switzerland}
\email{gleb.smirnov@unige.ch}

\author{Simon Karpenko}
\address{Institute for Information Transmission Problems, RAS, 
Bolshoi Karetnyi 19, 127051, Moscow, Russia}
\email{simon.karpenko@gmail.com}

\begin{abstract}
Fast Hough transform is a widely used algorithm in pattern recognition. The algorithm relies on approximating lines using a specific discrete line model called dyadic lines. The worst-case deviation of a dyadic line from the ideal line it used to construct grows as $O(log(n))$, where $n$ is the linear size of the image. But few lines actually reach the worst-case bound. The present paper addresses a statistical analysis of the 
deviation of a dyadic line from its ideal counterpart. Specifically, our findings show that the mean deviation is zero, and the variance grows as $O(log(n))$. As $n$ increases, the distribution of these (suitably normalized) deviations converges towards a normal distribution with zero mean and a small variance. This limiting result makes an essential use of ergodic theory.
\end{abstract}

\maketitle

\setcounter{section}{0}
\section{Dyadic lines}

The Radon transform (RT) maps a function $f$ defined on the plane $\rr^2$ to a function $R\! f$ defined on the (two-dimensional) space of lines in the plane. Specifically, if $f$ is function on $\rr^2$, $\Omega \subset \rr^2$ is the support of $f$ (image space), and $L \subset \rr^2$ is a line, then $R\! f(L)$ is given by the integral:
\[
R\! f(L) = \int_{L} f | d\,x |.
\]
In the context of digital image processing, a discrete form of the Radon transform is used. Various algorithms exist as discrete approximations of the continuous Radon transform. This paper focuses on one such discrete Radon transform, known as the Fast Hough transform (FHT). The FHT was originally proposed by Brady and Young \cite{Brady-Y}, independently by Vuillemin \cite{Vuill}, and also by G\"{o}tz-Druckm\"{u}ller \cite{Gtz-Druck}). The core idea behind FHT is to represent digital straight lines recursively, making it a highly efficient algorithm. Due to recursion, the total computational 
complexity is on the order $O(n^2 log(n))$, where $n$ is the size of the input image. Remarkably, the algorithm is suited for 
parallel computations and can be executed in $O(log(n))$ time using 
$O(n^2)$ processors. In this paper, we only focus on a key aspect of FHT: a specfic 
set of digital lines, referred to as dyadic lines (the term is borrowed from \cite{Ersh}). For a more detailed exploration of the algorithm itself, we refer the reader to \cite{Gtz-Druck} and \cite{Press}.
\smallskip%

It is important to note that our discussion here specifically pertains to the FHT algorithm introduced by Brady et al. However, it is worth mentioning that other algorithms, also bearing the name FHT, have been developed 
by H. Li, Lavin, Le Master, and Y. Li, Gan, as documented in \cite{Li, Li-Gan}. In this paper, we do not delve into the statistical analysis of these alternative algorithms, but we anticipate addressing them in future research.
\smallskip%

An image of linear size $n$ is represented as an $n \times n$-array $I$ of pixels, 
with indices ranging from $0$ to $n - 1$. Each pixel carries a non-negative 
number $f_{ij}$ representing the level of gray at that pixel. We define the 
dyadic line $D(t,x)$ connecting the origin $(0, 0)$ with the point $(n - 1, t)$ at the right border of the image. This dyadic line serves as a discrete 
approximation of the ideal line $y = t x/(n-1)$. For now, let us 
consider the case $0 \leq t \leq n - 1$, corresponding to the slopes from $0$ to $\pi/4$ (inclusive). We assume 
that $n = 2^p$ for some non-negative integer $p$. Instead of the original definition of 
$D(x, t)$, which involves a recursive process with dyadic lines on smaller images, we employ the novel analytic definition of $D(x,t)$ suggested in \cite{Ersh}. 
\smallskip%

First, we define basic dyadic lines, corresponding to the slopes 
$1, 2,\ldots, 2^{p-1}$. For each $i = 0, 1, \ldots, p - 1$, we set:
\begin{equation}\label{basic}
D(x, 2^{i}) = \left[ \dfrac{2^i\,x}{2^p - 1} \right],
\end{equation}
where $[\phantom{x}]$ represents rounding to the nearest integer.
\begin{center}
\begin{tikzpicture}[scale=0.7]

  \draw[gray!30] (0,0) grid (7,7);
  
  \draw[->] (0,0) -- (7,0) node[right] {$x$};
  \draw[->] (0,0) -- (0,7) node[above] {$t$};
  
  \draw[blue, thick] (0,0) -- (3,0) -- (4,1) -- (7, 1);
  \draw[blue] (3,5) node[above right] {$y = D(x, 1)$};
\end{tikzpicture}
\begin{tikzpicture}[scale=0.7]

  \draw[gray!30] (0,0) grid (7,7);
  
  \draw[->] (0,0) -- (7,0) node[right] {$x$};
  \draw[->] (0,0) -- (0,7) node[above] {$t$};
  
  \draw[blue, thick] (0, 0) -- (1,0) -- (2,1) -- (5,1) -- (6, 2) -- (7,2);
  \draw[blue] (3,5) node[above right] {$y = D(x, 2)$};
  
\end{tikzpicture}
\begin{tikzpicture}[scale=0.7]

  \draw[gray!30] (0,0) grid (7,7);
  
  \draw[->] (0,0) -- (7,0) node[right] {$x$};
  \draw[->] (0,0) -- (0,7) node[above] {$t$};
  
  \draw[blue, thick] (0, 0) -- (1,1) -- (2,1) -- (3,2) -- (4, 2) -- (5,3) -- (6, 3) -- (7, 4);
  \draw[blue] (3,5) node[above right] {$y = D(x, 4)$};
\end{tikzpicture}

\smallskip%

$n = 8$. Basic dyadic lines.
\end{center}

If we let $t_{p-1} t_{p-2} \ldots t_{0}$ be the binary 
representation of $t$, then $D(x, t)$ is defined as follows:
\begin{equation}\label{dyadic}
D(x, t) = \sum_{i = 0}^{p-1} t_i\, D(x, 2^{i}).
\end{equation}
\begin{center}
\begin{tikzpicture}[scale=0.7]
\node[text width = 9cm] at (-7.5,3.5) 
  {$n = 8$. dyadic lines $D(x, 7)$ and $D(x, 6)$.\\ 
  \smallskip%

  \noindent
  The two lines share 
  four pixels, which constitute a smaller line connecting 
  $(0,0)$ and $(3,3)$. When computing the sums of $f_{ij}$ for each line, FHT 
  saves time by calculating the shared partial sums only once.};

  \draw[gray!30] (0,0) grid (7,7);
  
  \draw[->] (0,0) -- (7,0) node[right] {$x$};
  \draw[->] (0,0) -- (0,7) node[above] {$t$};
  
  \draw[blue, thick] (0,0) -- (7,7) node[below right] {$y = D(x, 7)$};
  
  \draw[red, thick] (3, 3) -- (4,3) -- (7,6) node[below right] {$y = D(x, 6)$};
  
  \fill[black] (0,0) circle (2pt) node[below left] {$(0,0)$};
  \fill[black] (7,7) circle (2pt) node[above right] {$(7,7)$};
\end{tikzpicture}
\end{center}
In order to approximate a line with intercept, $y = t x/(n-1) + h$, we use the the dyadic line given by $D(t, x) + h$. If $h \neq 0$, then some lines fall off the top or bottom edges of the image. In order to handle those, the image is appropriately padded with zeros. Further, one defines dyadic lines for other slopes by appropriately flipping the original image.
\smallskip%

For each dyadic line on the image, FHT calculates 
the sum of $f_{ij}$ over the array points of that dyadic line. To assess the accuracy of FHT, we must understand how well a dyadic line approximates its ideal counterpart.

\section{Main results} 
The objective of this paper is to investigate the deviation of dyadic lines from their corresponding ideal lines. To simplify the exposition, we restrict ourselves to the pencil of dyadic lines passing through the origin. As before, we let $D(x,t)$ denote the dyadic line connecting 
the origin $(0,0)$ with the point $(n - 1, t)$, where $n = 2^{p}$ is 
the size of our digital image.  The deviation of $D(x,t)$ from 
the ideal line $y = t x/(n-1)$ at the pixel $(x,t)$ is expressed as:
\[
E(x,t) = D(x,t) - \dfrac{t x}{n - 1}
\]
Empirical studies (see \cite{Brady-Y, Gtz-Druck, Ersh}) indicate that the 
worst case deviation does not exceed $p/6$. Karpenko and Ershov provided a combinatorial 
proof of this estimate in \cite{Ersh-Karp}, confirming the following result: 
\begin{proposition}\label{up-bound}
$|E(x, t)| \leq p/6$, and the bound is sharp if $p$ is even.
\end{proposition}
In this note, we offer a novel proof of this estimate, 
relying on properties of circulant matrices. See \S\,\ref{sec-max}. 
Moreover, empirical evidence, as demonstrated in \cite{Gtz-Druck}, suggest that 
small deviations are more likely to occur than large ones. 
As shown in \cite{Ersh-Karp}, $E(x,t)$ obeys the following symmetry:
\begin{equation}\label{symmetry}
E(x, t) = - E(2^p - 1 - x, t).
\end{equation}
Consequently, when we randomly select points $(x, t)$ from the image, 
the mean of $E(x, t)$ is equal to $0$. However, knowing that $E(x, t)$ has 
an expected value of $0$ 
is insufficient to estimate how 
it deviates from zero; additional information, such as the variance, is required. 
In \S\,1 below, we calculate the variance $E(x, t)$ by means 
of simple algebraic arguments, and we establish:
\begin{proposition}\label{var}
In a uniform distribution on the 2D image $I$ of linear size $n = 2^p$, when each 
point $(x, t)$ of $I$ has an equal chance of being selected, the expected 
value of the function $E(x,t)$ equals zero and the variance is expressed as:
\begin{equation}\label{48var}
\mathrm{Var} E(x,t) = \dfrac{p}{48}\left( 1 - \dfrac{1}{2^p - 1} \right).
\end{equation}
\end{proposition}
With a variance as small as that in \eqref{48var}, 
even the basic Markov's inequality can provide useful estimates. 
For instance, we obtain:
\[
\mathbb{P}\left[ E(x,t)^2 \geq 1 \right] < \dfrac{p}{48}.
\]
Consequently, for $p = 12$, more than 75\% of points have a deviation less than one pixel. 
\smallskip%

We also 
provide additional insights into the distribution of $E(x, t)$ beyond just its mean and variance. Specifically, we establish the following statement:
\begin{proposition}\label{norm}
As $p \to \infty$, the distribution of $E(x, t)$ approaches normality, i.e., for each $a \in \rr$, 
\[
\lim_{p \to \infty} \sup_{a} \left| 
\mathbb{P} \left[ \sigma^{-1} E(x, t)/\sqrt{p} < a \right] - 
\Phi(a)
\right| = 0,\quad \sigma^{-1} = \sqrt{48}, 
\]
where $\Phi(a)$ is the standard normal cdf evaluated at $a$. 
\end{proposition}
This result follows from the Central Limit Theorem for dynamical systems, 
and we provide a sketch of the proof in \S\,3. 
However, it is important to note that until a thorough analysis of 
the convergence rate is conducted, 
the practical utility of this statement remains uncertain. 
\smallskip%

All results above rely 
on the following analytic representation of $E(x, t)$, 
as suggested in \cite{Ersh-Karp}. 
Using the definition of $D(x, t)$ provided 
in \eqref{dyadic}, we can calculate:
\[
E(x, t) = D(x,t) - \dfrac{t x}{2^p - 1} = 
\sum_{i = 0}^{p-1} t_i D(x, 2^{i}) - \sum_{i = 0}^{p-1} t_i 
\left( \dfrac{2^{i}\, x}{2^{p} - 1} \right) \stackrel{\eqref{basic}}{=} 
\sum_{i = 0}^{p-1} t_{i} 
\left( \left[ \dfrac{2^{i}\, x}{2^p - 1} \right] - \dfrac{2^{i}\, x}{2^p - 1} \right).
\]
Now, let us define:
\[
E_i(x) = \left[ \dfrac{2^{i}\, x}{2^p - 1} \right] - \dfrac{2^{i}\, x}{2^p - 1}.
\]
We can express $E(x, t)$ as follows:
\[
E(x, t) = \sum_{i = 0}^{p - 1} t_i\,E_{i}(x, t).
\]
If we write:
\[
\dfrac{x}{2^p - 1} = 0.(x_{p-1} x_{p-2} \ldots x_{0}),
\]
where the right-hand side represents a repeating binary fraction, 
then $E_{0}(x)$ takes 
the form: 
\begin{equation}\label{E0}
E_{0}(x) = 
\left[ 0.(x_{p-1} x_{p-2} \ldots x_{0}) \right] - 0.(x_{p-1} x_{p-2} \ldots x_{0}) = 
x_{p-1} - 0.(x_{p-1} x_{p-2} \ldots x_{0}).
\end{equation}
Let $P$ be the circular shift of the binary number $x_{p-1} x_{p-2} \ldots x_{0}$ that 
moves the last bit to the first position while shifting 
all other bits to the next position, i.e.,
\[
P(x_{p-1} x_{p-2} \ldots x_{0}) = x_{p-2} x_{p-3} \ldots x_{p-1}.
\]
Using the operation $P$, we express $E_i(x)$ as follows:
\begin{equation}\label{Ei}
E_{i}(x) = \left[ 2^{i}\,\cdot\,0.(x_{p-1} x_{p-2} \ldots x_{0}) \right] - 
2^{i}\,\cdot\, 0.(x_{p-1} x_{p-2} \ldots  x_{0}) = E_0(P^{i}(x)).
\end{equation}
Here, $P^{i}$ is the $i$th power of $P$.

\section{Proof of Proposition \ref{up-bound}}\label{sec-max} 
Similar to the proof by Karpenko and Ershov, 
we begin with equations \eqref{E0} and \eqref{Ei} and state 
this problem as a combinatorial optimization problem 
that involves minimizing a function defined on $p$-bit binary strings. 
However, we then follow a different route and cast this problem as a 
continuous optimization problem. 
\smallskip%

Note that substituting $2^{p} - 1 - x$ for $x$ changes the sign of $E(x, t)$ 
while maintaining its absolute value. Therefore, it suffices to show 
that $E(x,t) \leq p/6$, and that the bound is sharp for $p$ even. To begin with, we aim to maximize 
the values of $E(x, t)$ for a given $x$, 
finding the optimal value of $t$ that achieves this maximum. 
If we express $E(x, t)$ as: 
\[
E(x, t) = \sum_{i = 0}^{p - 1} t_i\,E_{i}(x),
\]
then the optimal choice of $t$ is as follows. For each $i = 0, \ldots, p -1$,  
\[
t_i = \begin{cases}
0 & \text{if } E_i(x) < 0 \\
1 & \text{if } E_i(x) \geq 0
\end{cases},
\]
which, after using \eqref{E0} and \eqref{Ei}, becomes $t_i = x_{p-1 - i}$. Therefore, it 
suffices to show the function
\begin{multline}
q(x) = \sum_{i = 0}^{p-1} x_{p-1 - i}\,E_i(x) = \sum_{i = 0}^{p - 1} x_{p-1 - i}\, 
\left( x_{p-1 - i} - 0.(x_{p-1 - i} x_{p-1-i-1}\ldots) \right)\\
= 
\sum_{i = 0}^{p - 1} x_{p-1 - i}\, 
\left( 1 - 0.(x_{p-1 - i} x_{p-1-i-1}\ldots) \right) = 
\sum_{i = 0}^{p - 1} x_{i} - 
\sum_{i = 0}^{p - 1} x_{p-1 - i} \cdot 0.(x_{p-1 - i} x_{p-1-i-1}\ldots)
\end{multline}
is upper-bounded by $p/6$, and that the bound is sharp if $p$ even. 
\smallskip%

Let us treat $x_{p-1} x_{p-2} \ldots x_{0}$ as vector 
$x = (x_0, x_1, \ldots, x_{p-1}) \in \rr^p$. 
Let $\ell$ be the 
linear function on $\rr^p$ defined as:
\[
\ell(x) = \sum_{i = 0}^{p-1} 2^{i} x_i.
\]
If the coordinates of the vector $x$ as 1's and 0's, then 
$(2^{p}-1)^{-1} \ell(x)$ calculates the decimal value of $0.(x_{p-1 - i} x_{p-1-i-1}\ldots)$. 
\smallskip%

From now on, we use $x$ to denote a vector in $\rr^n$ rather than the decimal value of 
$x_{p-1} x_{p-2}\ldots x_{0}$.

\smallskip%
Let $P \colon \rr^p \to \rr^p$ be the linear 
transformation that cyclically permutes the coordinates, i.e., 
\begin{equation}\label{P-lin}
P\colon 
(x_0, x_1, \ldots, x_{p-1}) \to (x_{p-1}, x_0, \ldots, x_{p-2}).
\end{equation}
The corresponding 
matrix, also denoted by $P$, 
is a cyclic permutation matrix with the structure:
\[
\begin{bmatrix}
  0 & 0 & 0 & \cdots & 1 \\
  1 & 0 & 0 & \cdots & 0 \\
  0 & 1 & 0 & \cdots & 0 \\
  \vdots & \vdots & \vdots & \ddots & \vdots \\
  0 & 0 & 0 & \cdots & 0 \\
\end{bmatrix}.
\]
Using $\ell$ and $P$, 
we extend $E_{i}(x)$ and $q(x)$ 
from the set binary strings to the entire space $\rr^p$. We get:
\[
q(x) = \sum_{i = 0}^{p - 1} x_{i} - 
(2^{p}-1)^{-1} \sum_{i = 0}^{p - 1} x_{p-1 - i} \ell(P^{i} x).
\]
After the change of variables $y_i = x_i - 1/2$, we get:
\[
q(y) = \dfrac{p}{4} - (2^{p}-1)^{-1} \sum_{i = 0}^{p - 1} y_{p-1 - i} \ell(P^{i} y) = 
\dfrac{p}{4} - (2^{p}-1)^{-1} a(y).
\]
Let us consider the unit hypercube in $\rr^p$, centered at $y = 0$, whose vertices have the $y$-coordinates: 
\[
\left( \pm \dfrac{1}{2}, \pm \dfrac{1}{2}, \ldots , \pm \dfrac{1}{2} \right). 
\]
Let $V$ denote the set of these vertices. To complete the proof, it suffices 
to show that 
\[
\min_{y \in V} a(y) \geq (p/12) \left( 2^p - 1 \right), 
\]
with equality when $p$ is even. Let 
$S$ be the unit sphere in $\rr^p$ centered at $y = 0$. 
Since $a$ is a quadratic form, 
and every point of $V$ 
lies on the sphere of radius $\sqrt{p}/2$, it follows that 
\[
\min_{y \in V} a(y) \geq (p/4) \min_{y \in S} a(y).
\]
Let us show that the minimum value of $a$ on $S$ is indeed 
$\left( 2^p - 1 \right)/3$ and that if $p$ even, the minimum point belongs 
to $V$. To this end, we use the following identity:
\[
a(y) = \langle A\,y, y \rangle,
\]
where $\langle \phantom{a}, \phantom{a} \rangle$ stands 
for the standard inner product on $\rr^n$, and the matrix $A$ is given by
\[
A = \begin{bmatrix}
  2^{p-1} & 1 & \cdots & 2^{p-2} \\
  2^{p-2} & 2^{p-1} & \cdots & 2^{p-3} \\
  2^{p-3} & 2^{p-2} & \cdots & 2^{p-4} \\
  \vdots & \vdots & \ddots & \vdots \\
  1 & 2 & \cdots & 2^{p-1} \\
\end{bmatrix}.
\]
This is indeed so, but the calculations are lengthy and 
we omit them. 
\smallskip%

If $A$ were symmetric, then the 
minimum value of 
$\langle A\,y, y \rangle$ on the unit sphere $S$ would be 
precisely the minimum eigenvalue of $A$, with the 
minimizer being the corresponding eigenvector of $A$. In our case, when 
$A$ is not symmetric, we apply the same argument to the symmetric matrix 
$(A + A^{T})/2$, which corresponds to the same quadratic form. 
\smallskip%

The matrix $A$ is a circulant matrix (see, e.g., \cite{Kra}) with the following structure:
\begin{equation}\label{poly}
A = 2^{p-1}\, P^{0} + 2^{p-2}\, P + 2^{p-3}\, P^2 + \ldots + 2^{0}\, P^{p-1}.
\end{equation}
It is known that the 
eigenvalues of $P$ are the $p$th roots of unity. 
If $\lambda$ is an eigenvalue of $P$ and 
$v$ is the corresponding (complex) 
eigenvector of $P$, then $v$ is also 
an eigenvector of $A$, associated with the eigenvalue:
\begin{multline}
2^{p-1} + 2^{p-2}\,\lambda + 2^{p-3}\,\lambda^2 + \ldots + \lambda^{p-1} = 
2^{p-1} 
\left( 1 + \dfrac{\lambda}{2} + \dfrac{\lambda^2}{4} + \ldots + 
\dfrac{\lambda^{p-1}}{2^{p-1}}  \right) = \\
 = \dfrac{2^{p} - \lambda^p}{2 - \lambda}  = 
\dfrac{2^{p} - 1}{2 - \lambda} = (2^{p} - 1) f(\lambda).
\end{multline}
On the other hand, if $A$ is circulant, then so is $A^t$. If 
$v$ is an eigenvector of $A$, associated with the eigenvalue $(2^{p} - 1) f(\lambda)$, then 
$v$ is also an eigenvector of $A^{t}$, associated with 
$(2^{p} - 1) \overline{f(\lambda)}$. Consequently, 
the eigenvalues of $(A + A^{T})/2$ all share the form 
$(2^{p} - 1) ( f(\lambda) + \overline{f(\lambda)} )$, where 
$\lambda^p = 1$. We seek 
to determine the minimal eigenvalue, i.e.,
\[
\min_{\lambda^p = 1} (2^{p} - 1) ( f(\lambda) + \overline{f(\lambda)} ) = 
(2^{p} - 1) \min_{\lambda^p = 1}  \mathrm{Re}\, f(\lambda).
\]
We find:
\begin{equation}\label{f}
\min_{\lambda^p = 1}  \mathrm{Re}\, f(\lambda) \geq  
\min_{|\lambda| = 1}  \mathrm{Re}\, f(\lambda) = f(-1) = 1/3.
\end{equation}
To bound $f$ on the unit circle, we argue as follows. 
Observe that $f$ is a M\"{o}bius transformation with real coefficients. 
As such, it sends the unit circle 
to the circle that passes through the points $f(-1) = 1/3$ and $f(1) = 1$ and 
that is also symmetric with respect to the 
reflection $\lambda \to \overline{\lambda}$. The leftmost point of this 
circle is indeed $1/3$. When $p$ is even, $(-1)$ is 
a $p$th root of unity, and the 
inequality in \eqref{f} becomes an equality. This completes the proof. \qed

\section{Proof of Proposition \ref{var}} 
To discuss the mean and variance of the function 
$E(x, t)$, we need to introduce a 
probability distribution over the image set $I$. 
A natural approach is to assign 
an equal probability of 
$2^{-p} \cdot 2^{-p}$ to each point $(x, y) \in I$, 
where $2^p$ is the linear size of $I$. For a number $t$, we consider its binary 
representation as $t_{p-1} t_{p-2} \ldots t_0$. 
Assuming that $t$ is uniformly distributed 
between $0$ and $2^p - 1$, each bit $t_i$ becomes a Bernoulli variable 
(akin to a fair coin) with a probability of $0.5$ of being either $0$ or $1$. 
Furthermore, if treated as random variables, 
these bits $t_i$ are independent of each other. 
The same properties of independence and fairness hold for the binary representation of $x$. 
We can now treat $E(x, t)$ as a random variable and proceed to calculate its first two moments, which are expressed as:
\[
\mathbb{E}\left[ E(x, t) \right] = 
\mathbb{E} \left[ \sum_{i = 0}^{p-1} t_i\,E_i(x) \right],\quad 
\mathbb{E}\left[ E(x, t)^2 \right] = 
\mathbb{E} \left[ \left( \sum_{i = 0}^{p-1} t_i\,E_i(x) \right)^2 \right].
\]
Since $t_i$ and $E_j(x)$ 
are mutually independent, these expressions simplify to:
\begin{equation}\label{moments-1}
\mathbb{E}\left[ E(x, t) \right] = 
(1/2)\,\mathbb{E} \left[ \sum_{i = 0}^{p-1} E_i(x) \right],\quad 
\mathbb{E}\left[ E(x, t)^2 \right] = 
(1/4)\,\mathbb{E} \left[ 
\sum_{i = 0}^{p-1} E_i(x)^2 + \left( \sum_{i = 0}^{p-1} E_i(x) \right)^2
\right].
\end{equation}
Let us again treat $x_{p-1} x_{p-2} \ldots x_{0}$ as vector 
$x = (x_0, x_1, \ldots, x_{p-1}) \in \rr^p$. 
Let $P \colon \rr^p \to \rr^p$ be defined by \eqref{P-lin}, and 
$\ell$ be defined 
as follows:
\[
\ell(x) = x_{p-1} - (2^p - 1)^{-1}\sum_{i = 0}^{p-1} 2^{i} x_i.
\]
This definition of $\ell$ differs 
from the one used in \S\,\ref{sec-max}. Setting 
\[
E(x, t) = \sum_{i = 0}^{p-1} t_i\,\ell(P^{i}\,x),
\]
we extend $E(x, t)$ to the function on the entire $\rr^p$. Then 
\eqref{moments-1} becomes:
\begin{equation}\label{moments-2}
\mathbb{E}\left[ E(x, t) \right] = 
(1/2)\,\mathbb{E} \left[ \sum_{i = 0}^{p-1} \ell(P^{i}\,x) \right],\quad 
\mathbb{E}\left[ E(x, t)^2 \right] = 
(1/4)\,\mathbb{E} \left[ 
\sum_{i = 0}^{p-1} \ell(P^{i}\,x)^2 + \left( \sum_{i = 0}^{p-1} \ell(P^{i}\,x) \right)^2
\right],
\end{equation}
where $x$ is uniformly distributed 
among the vectors whose coordinates are $1$'s and $0$'s. 
Let us show that for each $x \in \rr^p$, 
\begin{equation}\label{dr-vanish}
\sum_{i = 0}^{p-1} \ell(P^{i}\,x) = 0.
\end{equation}
To this end, we observe the left-hand side of 
\eqref{dr-vanish} is a linear 
function on $\rr^p$ that is also $P$-invariant. Therefore, we mush have 
\begin{equation}\label{cyclic}
\sum_{i = 0}^{p-1} \ell(P^{i}\,x) = \alpha \sum_{i = 0}^{p-1} x_i\quad 
\text{for some $\alpha \in \rr$.}
\end{equation}
On the other hand, we observe that 
\begin{equation}\label{l1}
\ell(\mathbf{1}) = 0,\quad \text{where 
$\mathbf{1} = (1, 1, \ldots, 1)$.}
\end{equation}
It is clear that for each $i = 0, \ldots, p - 1$, 
$\ell(P^{i}\,\mathbf{1}) = \ell(\mathbf{1}) = 0$. 
Combining this with \eqref{cyclic}, we find that $\alpha = 0$. Note 
that we also get:
\[
\sum_{i = 0}^{p-1} \ell(P^{i}\,x) = - \sum_{i = 0}^{p-1} \ell(P^{i}(\mathbf{1} - x)),
\]
which implies equation \eqref{symmetry}.
\smallskip%

Having established \eqref{dr-vanish}, we write \eqref{moments-2} as follows:
\[
\mathbb{E}\left[ E(x, t) \right] = 0,\quad 
\mathbb{E}\left[ E(x, t)^2 \right] = 
(1/4)\,\mathbb{E} \left[ 
\sum_{i = 0}^{p-1} \ell(P^{i}\,x)^2 
\right]
\]
Since $\ell(P^{i}\,x)$ are identically distributed, we get:
\begin{equation}\label{moments-3}
\mathbb{E}\left[ E(x, t)^2 \right] = 
(p/4) \mathbb{E}\left( \ell(x)^2 \right).
\end{equation}
To proceed, we use the following statement:
\begin{lemma}\label{trace}
Let $a(x) = \langle A\,x, x \rangle$ be a quadratic form 
on $\rr^n$. Let us consider the 
unit hypercube in $\rr^p$, 
centered at $y = 0$, whose vertices have the $y$-coordinates: 
\[
\left( \pm \dfrac{1}{2}, \pm \dfrac{1}{2}, \ldots , \pm \dfrac{1}{2} \right), 
\]
and let $V$ denote the set of these vertices. 
In a uniform distribution on $V$, when each vertex 
has an equal chance of being selected, the expected 
value of $\langle A\,x, x \rangle$ is expressed as:
\[
\mathbb{E} \left[ \langle A\,x, x \rangle \right] = (1/4)\,\mathrm{tr}\,A.
\]
\end{lemma}
\begin{proof}
Let $\mathcal{R}$ be the group generated by the following reflections:
\[
R_i \colon \rr^p \to \rr^p,\quad 
R_i(x_0, \ldots, x_{i-1}, x_i, x_{i+1}, \ldots , x_{p-1} ) = 
(x_0, \ldots, x_{i-1}, -x_i, x_{i+1}, \ldots ,x_{p-1}).
\]
Since the action of $\mathcal{R}$ on $V$ 
is free and transitive, we get:
\begin{multline}
\mathbb{E} \left[ \langle A\,x, x \rangle \right] = 
|V|^{-1} \sum_{x \in V} \langle A\,x, x \rangle = 
(1/4)\,|\mathcal{R}|^{-1} \sum_{R \in \mathcal{R}} \langle A R\,\mathbf{1}, R\,\mathbf{1} \rangle = \\
= (1/4)\,|\mathcal{R}|^{-1} \sum_{R \in \mathcal{R}} \langle R^{-1} A R\,\mathbf{1}, \mathbf{1} \rangle = 
(1/4)\,\langle |\mathcal{R}|^{-1} \sum_{R \in \mathcal{R}} R^{-1} A R\,\mathbf{1}, \mathbf{1} \rangle = 
(1/4)\,\langle \overline{A}\,\mathbf{1}, \mathbf{1} \rangle.
\end{multline}
Since 
the quadratic form $\langle \overline{A}\, x, x \rangle$ is $\mathcal{R}$-invariant, we 
must have 
$\langle \overline{A}\, x, x \rangle = \lambda_0 x_0^2 + \ldots + \lambda_{p-1} x_{p-1}^2$ 
for some $\lambda_i \in \rr$. Consequently, 
$\langle \overline{A}\, \mathbf{1}, \mathbf{1} \rangle = \mathrm{tr}\, \overline{A}$. 
It remains to note that $\mathrm{tr}\, \overline{A} = \mathrm{tr}\, A$, and the lemma 
follows. \qed
\end{proof}
\smallskip%

We will apply Lemma \ref{trace} as follows: 
Set $e_0 = (1, 0, \ldots, 0)$ and $e_i = P^{i}\,e_0$. We calculate: 
\[
4\, \mathbb{E}\left( \ell(x)^2 \right) \stackrel{\eqref{l1}}{=} 
4\, \mathbb{E}\left( \ell(x - 0.5 \cdot \mathbf{1})^2 \right) = 
\sum_{i = 0}^{p-1} \ell(e_i)^2.
\]
The term on the right is 
precisely the trace of the 
quadratic form $\ell(x)^2$. It is easy to compute explicitly this sum:
\[
\sum_{i = 0}^{p-1} \ell(e_i)^2 = \dfrac{1}{3} \left( 1 - \dfrac{1}{2^p - 1} \right).
\]
Combining this with \eqref{moments-3}, 
we complete the proof of Proposition \ref{var}. \qed

\section{Proof of Proposition \ref{norm}} 
In the following discussion, we 
find it convenient to rescale the coordinate $x$ by a factor of $(2^p - 1)$. 
For each integer $x$ with a binary 
representation $x_{p-1} x_{p-2} \ldots x_{0}$, we map it to the interval $[0,1]$ as follows:
\[
x \to x (2^p-1)^{-1} = 0.(x_{p-1} x_{p-2} \ldots x_{0}).
\]
From now on, we use $x$ to denote a repeating binary fraction 
of period $p$ rather than the decimal value of $x_{p-1} x_{p-2}\ldots x_{0}$. 
We let $\Delta_p \subset [0,1]$ denote the set of repeating 
binary fractions of period $p$, i.e., the set of numbers:
\[
i (2^p - 1)^{-1}\quad \text{for $i = 0, 1, \ldots, 2^p - 1$.}
\]
We will now redefine $E_i(x)$ as functions operating on repeating binary fractions. 
Consider the function $f$ defined on $[0,1]$ as follows:
\[
f(x) = \begin{cases}
-x & \text{if } x < 0.5 \\
1 - x & \text{if } x \geq 0.5
\end{cases}.
\]
Then, we set:
\[
E_0(x) = f(x)\quad \text{for each $x \in \Delta_p$.}
\]
This definition agrees with equation \eqref{E0}. 
Next, let $T \colon [0,1] \to [0,1]$ be the doubling map, defined as:
\[
T(x) = 2\,x\,\mathrm{mod}\,1.
\]
For each $i = 1, \ldots, p - 1$, we set: 
\[
E_i(x) = f(T^i\,x) \quad \text{for each $x \in \Delta_p$.}
\]
This definition agrees with equation \eqref{Ei}. Then, $E(x, t)$ takes the form:
\[
E(x, t) = \sum_{i = 0}^{i = p - 1} t_i\,f(T^{i}\,x).
\]
We assume that $t_i$ are iid fair coins, and that $x$ is uniformly distributed 
on the set $\Delta_p$. Our goal is to demonstrate that the normalized sums: 
\[
(1/\sqrt{p}) (t_0\, f(x) + t_1\, f(T\,x) + \ldots + t_{p-1}\, f(T^{p-1}\,x))
\]
converge weakly to the normal distribution with zero mean. The variance 
of the limiting distribution is found with Proposition \ref{var}. 
Since the doubling map $T$ is ergodic (mixing, even), the 
convergence follows rather easily from the CLT for dynamical systems, 
as discussed, for instance, in \cite{Nagaev, Gouzel}. However, 
in contrast to the conventional CLT framework, 
each $f(T^i\,x)$ contributes to the overall sum only 
with a probability of $0.5$, the domain of the variable $x$ 
depends on the parameter $p$, and the function $f$ is not continuous. 
We shall address these issues one by one.
\smallskip%

It suffices to show that the corresponding characteristic functions: 
\[
\psi_p(\xi) = \mathbb{E}\left( \exp\left\{ (i\, \xi/\sqrt{p}) (t_0\, f(x) + t_1\, f(T\,x) + \ldots + t_{p-1}\, f(T^{p-1}\,x)) \right\} \right) 
\]
converge pointwise to the characteristic 
function of a normal distribution. 
Firstly, we eliminate the dependency on $t_i$. Since they are mutually independent 
and independent of $x$, we get:
\begin{multline*}
\psi_p(\xi) = \mathbb{E}\left( 
2^{-p} \left( 1 + \exp\left\{ i\, \xi\, f(x)/\sqrt{p} \right\} \right) 
\left( 1 + \exp\left\{ i\, \xi\, f(T\,x)/\sqrt{p}  \right\} \right) \ldots 
\left( 1 + \exp\left\{ i\, \xi\, f(T^{p-1}\,x)/\sqrt{p} \right\} \right)
\right) = \\
= 2^{-p}\sum_{x \in \Delta_p}
2^{-p} \left( 1 + \exp\left\{ i\, \xi\, f(x)/\sqrt{p} \right\} \right) 
\left( 1 + \exp\left\{ i\, \xi\, f(T\,x)/\sqrt{p}  \right\} \right) \ldots 
\left( 1 + \exp\left\{ i\, \xi\, f(T^{p-1}\,x)/\sqrt{p} \right\} \right).
\end{multline*}
Secondly, we show that 
when interested solely in the limit of $\psi_p$, 
one can assume that $x$ is uniformly distributed 
on the entire interval $[0,1]$. 
More precisely, we establish the following result:
\begin{lemma}
Letting 
\[
u_p(x) = 2^{-p} \left( 1 + \exp\left\{ i\, \xi\, f(x)/\sqrt{p} \right\} \right) 
\left( 1 + \exp\left\{ i\, \xi\, f(T\,x)/\sqrt{p}  \right\} \right) \ldots 
\left( 1 + \exp\left\{ i\, \xi\, f(T^{p-1}\,x)/\sqrt{p} \right\} \right),
\]
we get:
\[
\left| \int_{0}^{1} u_p - 2^{-p} \sum_{x \in \Delta_p} u_p \right| < C |\xi|/\sqrt{p}
\]
for some constant $C$ independent of $p$ and $\xi$.
\end{lemma}
\begin{proof}
Let $\Sigma_p \subset [0,1]$ be the set of numbers:
\[
j \cdot 2^{-p}\quad \text{for $j = 1, \ldots, 2^p - 1$.}
\]
The complement $[0,1] - \Sigma_p$ consists of $2^p$ open 
intervals (the leftmost and rightmost intervals are half-open), each containing 
precisely one point of $\Delta_p$. Specifically, we have
\[
x_i = \dfrac{i}{2^p - 1} \in \left[ \dfrac{i}{2^p}, \dfrac{i + 1}{2^p} \right]\quad 
\text{for each $i = 0, 1,\ldots, 2^p - 1$.}
\]
By construction, for each $k = 0, 1, \ldots, p - 1$, the function 
$f \circ T^k$ is differentiable on $[0,1] - \Sigma_p$. Moreover, for each 
$i = 0, 1, \ldots, 2^{p} - 1$, $f \circ T^k$ is a linear 
function on the interval $(i \cdot 2^{-p}, (i+1) 2^{-p})$. 
In particular, $f \circ T^k$ extends to the boundary points 
of that interval as a differentiable function. With this understood, we write: 
\begin{multline*}
\left|
\int_{0}^{1} u_p - 2^{-p} \sum_{x \in \Delta_p} u_p 
\right| = 
\left|
\sum_{i = 0}^{2^p - 1} \int_{i/2^p}^{(i+1)/2^p} u_p - 
2^{-p} \sum_{i = 0}^{2^p - 1} u_p\left(x_i\right) 
\right| = 
\left|
\sum_{i = 0}^{2^p - 1} \int_{i/2^p}^{(i+1)/2^p} 
\left( u_p(x) - u_p(x_i) \right)
\right| \leq \\
\leq 
\sum_{i = 0}^{2^p - 1} \int_{i/2^p}^{(i+1)/2^p} 
|u_p(x) - u_p(x_i)|
\leq
\sum_{i = 0}^{2^p - 1} \int_{i/2^p}^{(i+1)/2^p} 
M\,2^{-p} = M\,2^{-p},
\end{multline*}
where $M = \sup_{x \in [0,1] - \Sigma_p} |\partial_x\,u_p|$. 
We estimate $M$ as follows:
\[
\partial_x\,u_p = 
2^{-p}\sum_{i = 0}^{p-1} i\,\xi\,\partial_x\,(f \circ T^i)\, 
\exp\left\{ i\,\xi\,f \circ T^k \right\} \prod_{i \neq k} 
\left( 1 + \exp\left\{ i\,\xi\,f \circ T^i \right\} \right).
\]
As both $\xi$ and $f$ are real, we obtain:
\[
|\partial_x\,u_p| \leq 2^{-p} 
\left( |\partial_x\,f| + |\partial_x\,(f \circ T)| + 
\cdots + |\partial_x\,(f \circ T^{p-1})| \right) 2^{p-1} |\xi|.
\]
For each $x \notin \Sigma_p$, the derivative of $f(T^i\,x)$ equals 
$(-2^{i})$. Consequently, we have $M \leq |\xi|(2^p - 1)/2$, and 
\[
\left|
\int_{0}^{1} u_p - 2^{-p} \sum_{x \in \Delta_p} u_p 
\right| \leq |\xi| (1 - 2^{-p})/2 < |\xi|.
\]
This finishes the proof. \qed
\end{proof}
\smallskip%

It follows that pointwise limit of 
$\psi_p(\xi)$ is equal to that of the sequence:
\[
\varphi_p(\xi) = \int_{0}^{1} 2^{-p} \left( 1 + \exp\left\{ i\, \xi\, f(x)/\sqrt{p} \right\} \right) 
\left( 1 + \exp\left\{ i\, \xi\, f(T\,x)/\sqrt{p}  \right\} \right) \ldots 
\left( 1 + \exp\left\{ i\, \xi\, f(T^{p-1}\,x)/\sqrt{p} \right\} \right).
\]
The remaining part of the proof employs 
Nagaev's method, which is detailed 
in \cite{Nagaev, Gouzel}. 
For the sake of completeness, we will briefly outline the key steps of the method. Firstly, 
we introduce the transfer operator $L_0$ as follows: 
if $h$ is a sufficiently regular 
function on the interval $[0,1]$, 
possibly a bounded 
function with only finitely many discontinuities, then the function $L_0\,h$ is defined as:
\[
(L_0\,h)(x) = \sum_{y \in T^{-1}(x)} h(y)/|T^{-1}(x)| = 
\dfrac{1}{2} \left( h\left( \dfrac{x}{2}\right) + 
h\left( \dfrac{1}{2} + \dfrac{x}{2}\right) \right)\quad \text{for each $x \in (0,1)$},
\]
and then we extend $L_0\,h$ to $0$ and $1$ by continuity. 
This operator satisfies:
\begin{equation}\label{adjoint}
\int (L_0\,h) \cdot g = \int h \cdot (g \circ T)
\end{equation}
whenever both sides are defined. In particular, setting $g = 1$, we get:
\begin{equation}\label{meanL0}
\int L_0\,h = \int h.
\end{equation}
Next, let us 
define the twisted transfer operator $L_{\xi}$ as follows:
\begin{equation}\label{twisted}
L_{\xi}\,h = L_0\left( \dfrac{1}{2}(1 + \exp\left\{ i\,\xi\,f \right\}) h \right).
\end{equation}
Explicitly, using the identities:
\[
f\left( \dfrac{x}{2} \right) = -\dfrac{x}{2},\quad 
f\left( \dfrac{1}{2} + \dfrac{x}{2} \right) = 1 - \dfrac{x}{2}\quad 
\text{for each $x \in (0,1)$,}
\]
we can express $L_{\xi}\,h$ as:
\begin{equation}\label{explicit}
(L_{\xi}\,h)(x) = 
\dfrac{1}{4} \left( 
h\left( \dfrac{x}{2}\right) 
\left( 1 + \exp\left\{ i\,\xi\,\left( - \dfrac{x}{2} \right) \right\} \right) + 
h\left( \dfrac{1}{2} + \dfrac{x}{2}\right)
\left( 1 + \exp\left\{ i\,\xi\,\left( 1 - \dfrac{x}{2} \right) \right\} \right)
\right).
\end{equation}
We adopt this definition to ensure that Nagaev's identity holds true:
\[
\varphi_p(\xi) = \int 2^{-p} \left( 1 + \exp\left\{ i\, \xi\, f(x)/\sqrt{p} \right\} \right) 
\left( 1 + \exp\left\{ i\, \xi\, f(T\,x)/\sqrt{p}  \right\} \right) \ldots 
\left( 1 + \exp\left\{ i\, \xi\, f(T^{p-1}\,x)/\sqrt{p} \right\} \right) = 
\int L_{\xi}^p 1.
\]
This identity can be established by repeatedly applying \eqref{adjoint}. 
\smallskip%

For the next stage of the proof, 
we need to identify a suitable Banach space of functions, denoted as $\mathcal{B}$. 
In this space, $L_{\xi} \colon \mathcal{B} \to \mathcal{B}$ should 
be bounded for sufficiently small $\xi$, and 
the correspondence $\xi \to L_{\xi}$ should be smooth or, better yet, holomorphic. 
The function $1$ in $\mathcal{B}$ should be an 
eigenfunction for $L_0$ with the eigenvalue $1$ and that eigenvalue should be simple. 
Also, it should be ensured that all other eigenvalues of $L_0$ are contained within a disk in the 
complex plane with a radius strictly less than $1$.   
\smallskip%

One possible approach is 
to define $\mathcal{B}$ as the space of $C^1$-functions on $[0,1]$. 
Within $\mathcal{B}$, we can decompose functions into two subspaces, 
$\langle 1 \rangle$ and $\overline{B}$. Here, $\langle 1 \rangle$ represents the space of constant functions, and $\overline{B}$ is the space of functions with zero mean:
\[
\overline{B} = \left\{ h \in \mathcal{B}\ |\ \int h = 0  \right\}.
\]
The space $\langle 1 \rangle$ is an eigenspace 
for $L_0$ associated with the eigenvalue $1$, and this eigenvalue is indeed simple for $L_0$. 
Moreover, it follows form \eqref{meanL0} that $\overline{B}$ is an invariant 
subspace for $L_0$. One shows that $L_0 \colon \overline{B} \to \overline{B}$ 
acts as a contraction map with respect to the $C^1$-norm, 
confirming that the other eigenvalues do indeed lie witih disk of radius 
strictly less than $1$.
\smallskip%

That $L_{\xi}$ is a bounded operator on $\mathcal{B}$ does not 
immediately follow from its defining formula \eqref{twisted}. 
While $L_0$ is bounded, the multiplication operator 
$h \to (1 + \exp\left\{ i,\xi,f \right\}) h$ in \eqref{twisted} 
is not bounded due to the lack of continuity in $f$. 
Nevertheless, when this multiplication 
operator is combined with $L_0$, it yields a bounded operator, as shown 
by formula \eqref{explicit}. It is also 
apparent from \eqref{explicit} that the path $\xi \to L_{\xi}$ is analytic. 
\smallskip%

That the pointwise limit of 
\[
\int L_{\xi/\sqrt{p}}^p 1
\]
equals $\exp\left\{ - \sigma^2\,x^2 \right\}$ for some $\sigma^2 \geq 0$ is 
a result of Nagaev. For a detailed explanation, see, e.g, 
Theorem 2.4 in \cite{Gouzel}. The parameter $\sigma$ of the limiting distribution can be determined from Proposition \ref{var}. This finishes the proof. \qed

\smallskip
\section*{Acknowledgements} GS is supported by an SNSF Ambizione fellowship.

\smallskip

\bibliographystyle{plain}
\bibliography{ref}

\end{document}